\documentclass{article}

     \PassOptionsToPackage{numbers, compress}{natbib}

\usepackage[preprint]{neurips_2023}

\usepackage[utf8]{inputenc} 
\usepackage[T1]{fontenc}    
\usepackage{hyperref}       
\usepackage{url}            
\usepackage{booktabs}       
\usepackage{amsfonts}       
\usepackage{nicefrac}       
\usepackage{microtype}      
\usepackage{xcolor}         

\usepackage{xspace,graphicx,amsmath,bbm}
\usepackage{algorithm,algpseudocode,mathtools}

\usepackage{caption,subcaption,amsthm}

\newtheorem{theorem}{Theorem}[section]

\newtheorem{remark}[theorem]{Remark}
\newtheorem{example}[theorem]{Example}
\newtheorem{definition}[theorem]{Definition}
\newtheorem{corollary}{Corollary}[theorem]

\newcommand\hull{$\mathcal{H}^{tr}$\xspace}

\definecolor{purple}{cmyk}{0.65, .8, 0, 0}
\definecolor{cyan}{cmyk}{.9, 0.00, 0.0, 0.1}

\title{An Ambiguity Measure for Recognizing the Unknowns in Deep Learning}

\author{%
  Roozbeh Yousefzadeh\\
  Huawei Hong Kong Research Center\\
  roozbeh.yz@gmail.com
}

\begin{document}

\maketitle

\begin{abstract}
We study the understanding of deep neural networks from the scope in which they are trained on. While the accuracy of these models is usually impressive on the aggregate level, they still make mistakes, sometimes on cases that appear to be trivial. Moreover, these models are not reliable in realizing what they do not know leading to failures such as adversarial vulnerability and out-of-distribution failures. Here, we propose a measure for quantifying the ambiguity of inputs for any given model with regard to the scope of its training. We define the ambiguity based on the geometric arrangements of the decision boundaries and the convex hull of training set in the feature space learned by the trained model, and demonstrate that a single ambiguity measure may detect a considerable portion of mistakes of a model on in-distribution samples, adversarial inputs, as well as out-of-distribution inputs. Using our ambiguity measure, a model may abstain from classification when it encounters ambiguous inputs leading to a better model accuracy not just on a given testing set, but on the inputs it may encounter at the world at large. In pursuit of this measure, we develop a theoretical framework that can identify the unknowns of the model in relation to its scope. We put this in perspective with the confidence of the model and develop formulations to identify the regions of the domain which are unknown to the model, yet the model is guaranteed to have high confidence.
\end{abstract}


\section{Introduction}

How can a model know what it does not know? How can we develop models that can reliably relate inputs to their knowledge and abstain from processing inputs that are ambiguous to them? In many applications of deep learning, e.g., in object detection, we encounter inputs that are hard to process and evaluate, i.e., corner cases. Consider the case of autonomous driving where a model may encounter ambiguous images, e.g., a pedestrian with an unusual clothing, etc. These are potentially points that may lead to accidents and it would be useful to automatically quantify whether a model is uncertain about its decisions. Moreover, a model may encounter inputs that are adversarially designed to mislead it \cite{szegedy2014intriguing}. We also have cases of out-of-distribution inputs \cite{hendrycks2017baseline}. All these cases can be considered various types of input ambiguity for a given model, i.e., the unknowns of the model, potentially leading to its failure. 
If we have tools to detect such inputs when we encounter them, we can have the model make more conservative decisions or abstain all together. For example, an autonomous car could decide to slow down or stop by the side of the road, if it is unsure about what it sees in the road ahead.

\begin{figure}[h]
  \centering
   \includegraphics[width=.9\linewidth]{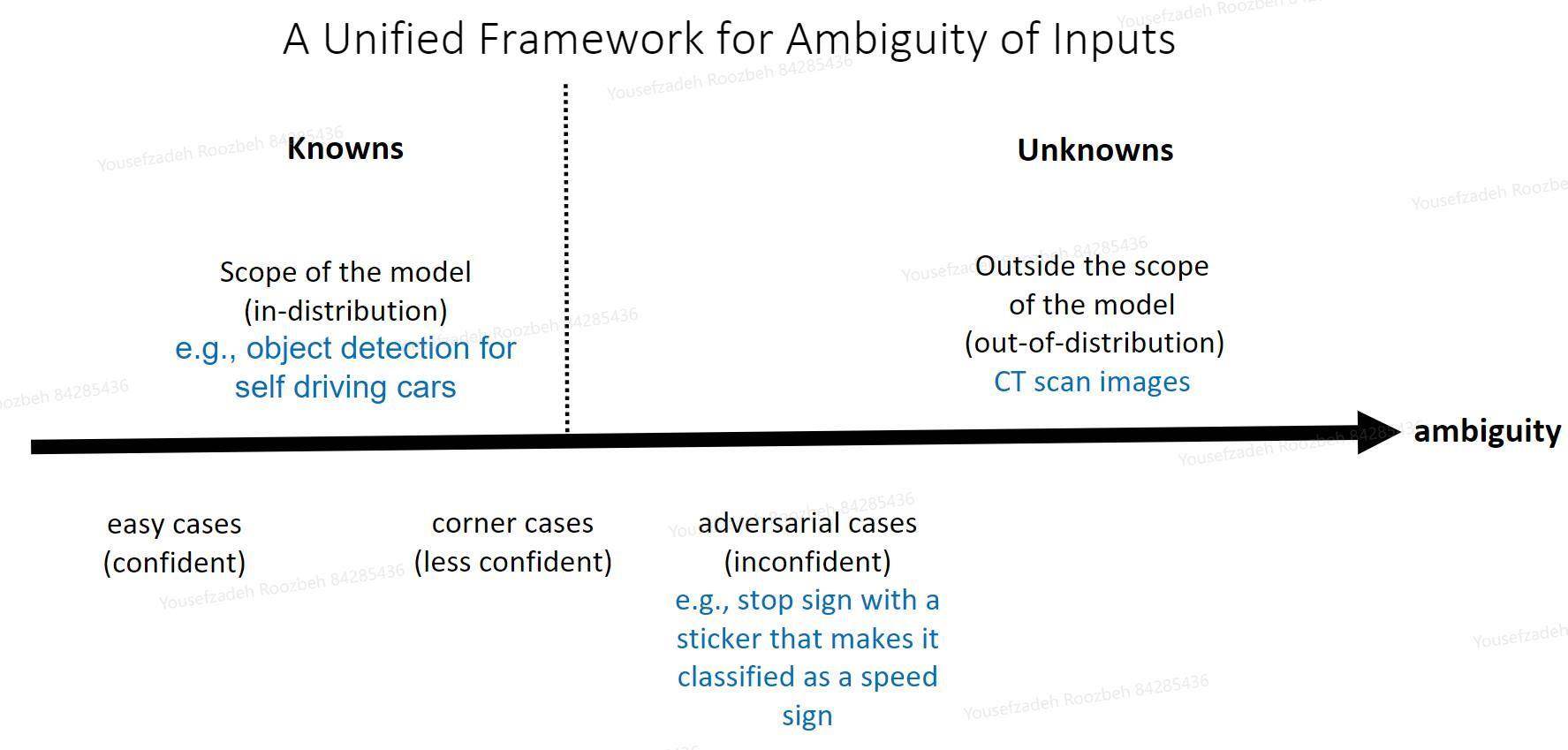}
   \caption{We propose a unified ambiguity framework to identify the knowns and unknowns of deep learning models.}
  \label{fig:amb_spectrum}
\end{figure}






The problem is that usually unknowns of a model is unknown to itself. In other words, AI models usually do not know what they do not know \cite{nalisnick2018do}. 
Indeed, the question of recognizing the unknowns of a model is central in many applications of AI models. In the literature, we have separate fields that study adversarial vulnerabilities \cite{goldblum2023dataset} and out-of-distribution generalization \cite{ye2022ood}. We also have studies that aim to improve the accuracy of models by reducing their mistakes on in-distribution testing samples, e.g., by model calibration \cite{Nixon_2019_CVPR_Workshops}, uncertainty quantification \cite{ovadia2019trust}, misclassification detection \cite{hendrycks2016baseline}. However, we do not have a unified framework to connect all of these failure modes. Imagine we want to deploy a trained model in the real world while concerned about all the above sorts of failures. It seems very expensive and impractical to deploy a model while having a pipeline of supplementary models for weeding out adversarial and out-of-distribution inputs along with corner cases. Often, models offered for adversarial or out-of-distribution detection are more expensive than the main model at hand \cite{ren2019likelihood}, e.g., the recent OOD detection methods require the images to be processed with a diffusion model, or many of the other OOD detection methods rely on ensembles of models as large as the original model. Hence, the pipeline of supplementary models is extremely expensive compared to the model itself. If we set aside the question of computational feasibility, another interesting question is whether these failure detection methods can work in harmony with each other to weed out the failure modes of a specific model at hand? The answer seems to be currently no, as \cite{jaeger2023call} demonstrated that most of these failure detection methods do not generalize well, and moreover, for some of these failure detection methods, even the evaluation metrics used in the literature are not appropriate.



To address these issues, we propose an ambiguity measure that can relate these seemingly separate fields of study into a unified framework by leveraging the geometric information available in the feature space that models learn from their training set, specifically, in relation to the decision boundaries and convex hulls, the information that is freely available to any model designer, but it is often neglected. Our contributions can be summarized as the following:

\begin{enumerate}
    \item We first develop a theoretical framework to evaluate the unknowns of the trained models in relation to their confidence. With certain guarantees we quantify the portions of the domain which are unknown to the trained models, yet the models are guaranteed to process them with high confidence.
    \item We propose an ambiguity measure that is based on the geometric arrangements in the feature space of a trained model and the common failure modes of the models. We use image classification models as our case study and provide the necessary algorithms to compute the ambiguity measure fast, in a fraction of a second for models as large as Swin Transformer trained on Imagenet \cite{deng2009imagenet}. Through experiments, we also show that our proposed measure is able to detect a considerable portion of the mistakes made by the models, along with adversarial inputs, and out-of-distribution images.
    \item Our ambiguity measure provides an automatically generated text describing why an image is ambiguous, relating the ambiguity to samples in the training set and to the decision boundaries of the model.
    \item We show that there is valuable geometric information in the feature space of deep networks which has been often neglected, especially with regard to the convex hull of training sets.
\end{enumerate}

\section{Geometric information available in the last hidden layer of trained models}

Let us consider a classification model $\mathcal{M}$ trained on a training set $\mathcal{D}^{tr}$. This model may have $d$ dimensions for its input space and $n$ dimensions for its output corresponding to $n$ distinct classes $$z = \mathcal{M}(x),$$ where $z$ is the output of the model for input $x$, and the element of $z$ with largest value will indicate the classification for $x$
$$\mathcal{C}(x) = \{i: z_i=\max_{k} z_k \}.$$

We use $\Phi$ to denote the latent space in the last hidden layer of a given model. This feature space has $f$ dimensions. For example, a typical Swin Transformer trained on Imagenet has $f=1,536$ and $n=1,000$ in its last hidden layer. When $\mathcal{M}$ processes an input $x$, it maps the $x$ to its latent space
\begin{equation} \label{eq:map_phi}
    x_\phi = \mathcal{M}_\phi(x),
\end{equation}
and then $x_\phi$ will be processed via
\begin{equation} \label{eq:output_phi}
   z = W_\phi x_\phi + b_\phi,
\end{equation}
where $W_\phi$ and $b_\phi$ are the weight matrix and bias vector of the last hidden layer of the model.

\subsection{Why study the learned feature space?}

Our main motivation in this approach is to understand and verify the learning of a given trained model, i.e., what the model has learned from its training data and how it uses its learning to process the new inputs.

The success of deep networks is largely attributed to the features they extract from their inputs. For images, contents of interest usually cannot be explained in terms of individual pixels. Rather, for both humans and the models, the spatial relationship between the groups of pixels is the key to identifying what is depicted in an image. Similarly, for natural language processing, the sequence of words matter in the meaning created from individual words. Another way to describe this phenomenon is that inputs, such as image and text (and many other types of data), have a lower dimensional structure, and in order to learn from such data, one has to learn that underlying structure. This has been a subject of study from various perspectives in machine learning often under the term representation learning \cite{bengio2013representation,cohen2020separability}. Studying the underlying structure of the data also has a rich literature in various fields of mathematics \cite{adler1967modern,golub2012matrix}.

In pursuit of understanding the underlying structure, it is important to note that deep networks often make mistakes, and their learning from the data sometimes has certain flaws, i.e., they may fail to learn certain structures in the data. Moreover, training sets may have certain biases and flaws of their own which may affect the learning of the models trained on them. For example, in a facial recognition dataset, certain groups of people may not be well represented and that could affect the feature space learned by the models \cite{jain2023distilling}.

Here, we study the learned feature spaces not as flawless manifolds that represent the ideal way of learning. Rather, we study a given feature space to view what a model has learned from its data, and to verify the knowns and unknowns of the given model. We provide systematic methods that can interpret the learning of a trained model, and explain how it views the testing samples in relation to what it has learned.





\subsection{Feature space: further compression via SVD}

The feature space that we consider is based on the Singular Value Decomposition (SVD) of $\Phi$. We use SVD to project the points in the latent space of the last hidden layer to a lower dimensional space denoted by $\Psi$. SVD decomposes the matrix $W_\phi$ into three matrices: $$W_\phi = U \Sigma V^T,$$ where $U$ is a unitary $n \times n$ matrix, $\Sigma$ is a diagonal $n \times f$ matrix, and V is a unitary $f \times f$ matrix. We use $V^T$ to project the points in the latent space: 
\begin{equation} \label{eq:map_psi}
    x_\psi = V^T x_\phi.
\end{equation}
Given a point $x_\psi$ in feature space $\Psi$, one can obtain the output of the model via
\begin{equation} \label{eq:psi_output}
   z = U \Sigma x_\psi + b_\phi.
\end{equation}

{\bf Benefits of decomposing the last hidden layer:} The results we present in this paper are all in the $\Psi$ space. We note that analyzing the original space of $\Phi$ yields to results that are slightly less accurate compared to the results obtained from $\Phi$. The better accuracy is explainable via the rotation and scaling operations of SVD \cite{golub2012matrix} which may pose the geometric arrangement of data in a way that is more suitable for our classification task. This is evident in Figure~\ref{fig:svd_disthull} where for the orange histogram ($\Psi$), the separation between the closest class and the other classes is more pronounced. Indeed, if we were to use the distance to the convex hull of closest class as our classification criteria, the accuracy of our model would be 94.40\% for $\Psi$, 94.35\% for $\Phi$, and 94.33\% for $\Phi @ U$ while the accuracy obtained from the softmax score is 94.37\%. This indicates that for accurate classification, one can rely merely on the convex hull of training set in the feature space without considering the decision boundaries of this model.


\begin{figure}[h]
  \centering
   \includegraphics[width=.9\linewidth]{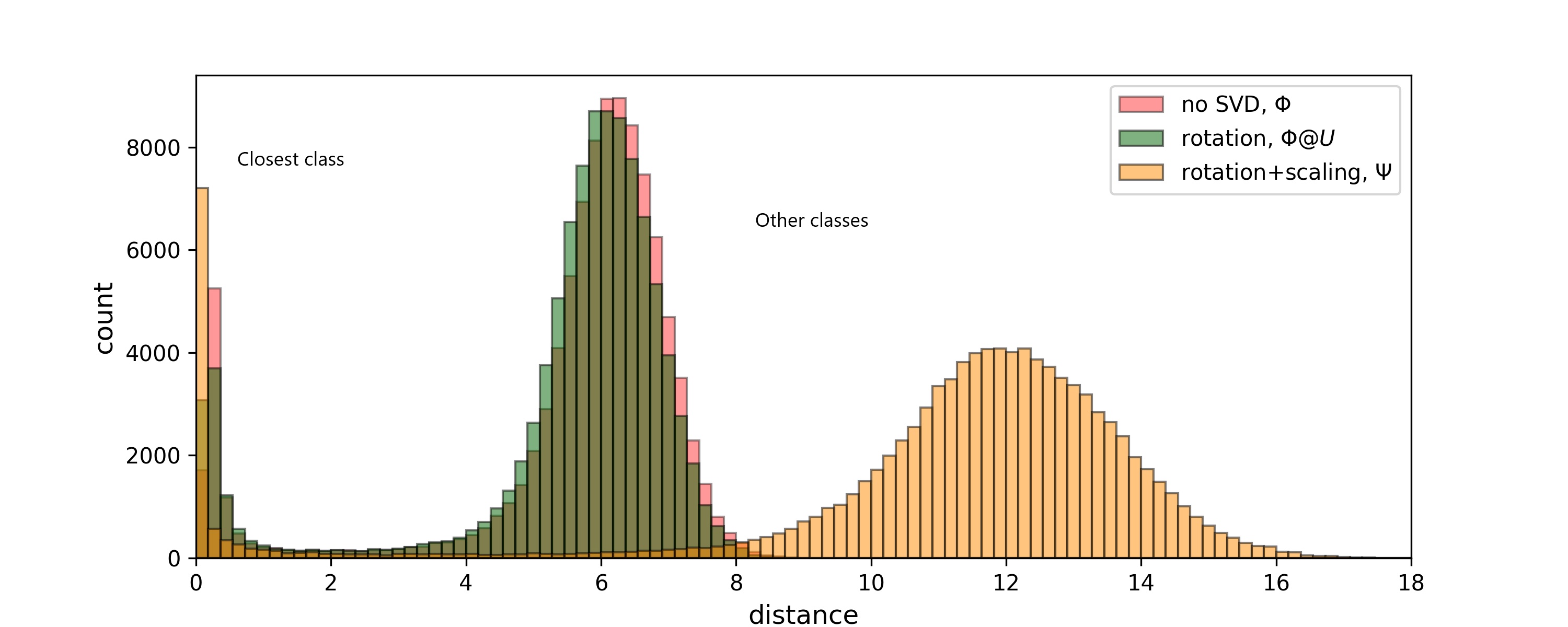}
   \caption{Distribution of distance of testing samples from the convex hull of training set for each class of CIFAR-10 dataset in the feature space learned by a ResNet model.}
  \label{fig:svd_disthull}
\end{figure}

The other benefit obtained from the SVD is that $\Psi$ has a smaller number of dimensions compared to~$\Phi$. For example, a Swin transformer has 1,536 dimensions in $\Phi$ while its corresponding $\Psi$ has only 1,000 dimensions. This makes our optimization problems, regarding the convex hulls and decision boundaries, less expensive. We note that SVD has been used in the last hidden layer of language models before, not to change the geometric arrangements, but to make the computation of softmax scores faster \cite{shim2017svd,chen2019learning}.

\subsection{The intuition behind our geometric approach}

Let us consider a deep network trained for classification of 3 classes. Figure~\ref{fig:2dpartition} depicts the feature space in the $\Psi$ space of such model which we extracted from 3 output classes of a Swin Transformer trained on ImageNet. This figure, not only shows the decision boundaries of the model in this feature space, it also depicts the convex hull of training samples for each class. When classifying a testing sample $x$, a typical model maps the sample to the $\Psi$ space, and then, determines which partition the $x_\psi$ belongs to. The partitioning step is performed via equation~\eqref{eq:psi_output} where partitions are separated by a finite set of hyperplanes \cite{yousefzadeh2022feature}.

\begin{figure}[h]
  \centering
   \includegraphics[width=.8\linewidth]{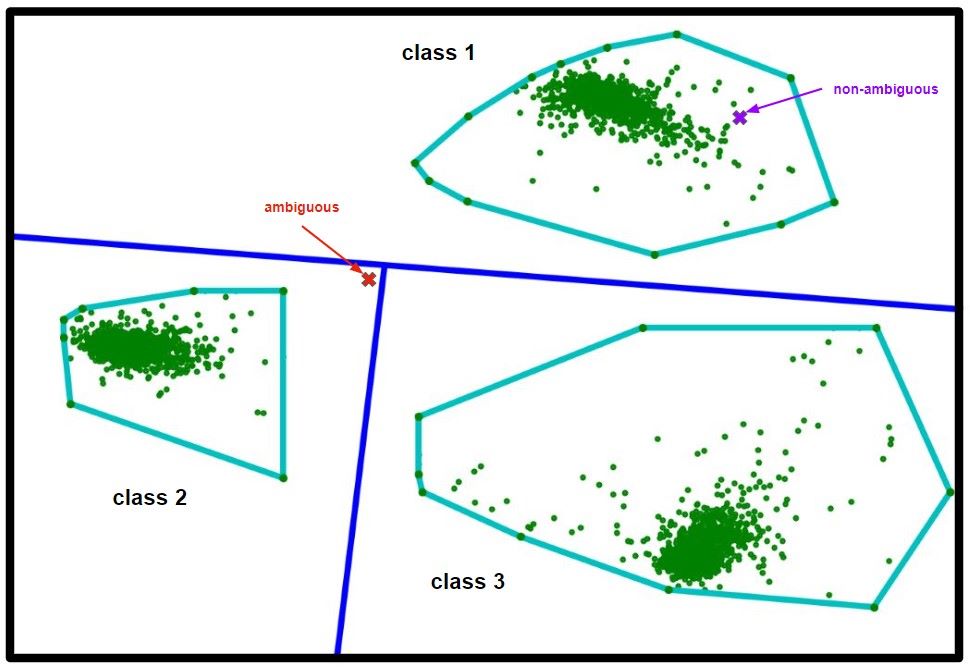}
   \caption{Partitioning of a 2D domain via \textcolor{blue}{decision boundaries} and the \textcolor{cyan}{convex hull} of training set for each class. }
  \label{fig:2dpartition}
\end{figure}

The linear operation performed via equation~\eqref{eq:psi_output} partitions the feature space via a set of hyperplanes. Since $W_\phi$ and $b_\phi$ are fixed for a trained model, the partitioning of the feature space (i.e., the location of decision boundaries) is also fixed. It follows that once we compute the $x_\psi$ for a given $x$, equation~\eqref{eq:psi_output} merely determines which partition the $x_\psi$ has fallen within. This, however, neglects the rich geometric information available in the feature space, most importantly the relationship between the $x_\psi$ and the convex hull of training set for each class. The other valuable information in the feature space is the arrangement with regard to the decision boundaries.

To make our point more clear, consider the \textcolor{red}{red} and \textcolor{purple}{purple} points in Figure~\ref{fig:2dpartition}, each of them being a testing sample. The \textcolor{purple}{purple} point falls within the convex hull of class 1 and it is relatively far from the decision boundaries. Being inside the convex hull indicates that this sample falls within the range of training samples for class 1, i.e., the model has some frame of reference for classifying this sample as class 1. Being away from the decision boundaries with other classes signifies that there is little chance of confusion on which partition/class this sample may belong to.

Now, let us consider the \textcolor{red}{red} point in Figure~\ref{fig:2dpartition} which falls in the partition for class 2. This sample is far from the convex hull of training set for class 2, and at the same time, it is very close to the decision boundaries with the other 2 classes. One can argue, quite reasonably, that the model should be less confident about the classification of the \textcolor{red}{red} point compared to the classification of the \textcolor{purple}{purple} point. While the red sample has fallen in the partition for class 2, it does not seem to be closely similar to any of the training samples for class 2. Yet, it can be easily flipped to any of the other classes, if it were slightly perturbed or if the decision boundaries of the model were slightly different than their current configuration.

These two concepts, the relationship with the convex hull of training set and the relationship with the decision boundaries, are the basis for our ambiguity measure.

\section{Computing the distances to the convex hulls and the decision boundaries}

\subsection{Convex hulls}
Assuming that training set has at least one sample for each class, one would be able to compute the distance to the convex hull of each class.

\begin{equation}
    x_\psi^{\mathcal{H},k} = \mathcal{P}^{\mathcal{H}}(x_\psi, \mathcal{H}_\psi^{tr},k),
\end{equation}
where $\mathcal{P}$ is a function that maps the $x_\psi$ to the convex hull of training set for class $k$. Using this projection, we can compute a vector of distances where each of its elements corresponds to one of the classes in training set

\begin{equation}
    d^{\mathcal{H}}_\psi(x) : \forall k \in \{1:n\} , \quad d^{\mathcal{H}}_\psi(x_\psi)[k] =  \| x_\psi - \mathcal{P}^{\mathcal{H}}(x_\psi, \mathcal{H}_\psi^{tr},k) \|_F,
\end{equation}

We will be particularly interested in the first two smallest values of this vector.

To compute the projection of a sample to the convex hull of a set of points, we use the approximation algorithm by \cite{har2016space,blum2019sparse} which has a guarantee for small error and a running time of $\mathcal{O}(|\mathcal{D}|\frac{n}{\bar{\epsilon}^2})$ where $|\mathcal{D}|$ is the number of points in the convex hull, $n$ is the number of dimensions, and $\bar{\epsilon}$ is a parameter to bound the error. We note that the problem of projecting to a convex hull is convex and there are exact algorithms that can solve it reasonably fast \cite{nocedal2006numerical}, but not as fast as the approximation algorithm.

In practice, solving this optimization problem in a 1,000 dimensional feature space of a Swin Transformer trained on Imagenet takes about 100 milliseconds on a GPU.

\subsection{Decision boundaries} \label{sec:dbs}


\subsubsection{Existence, intersections, and properties}  \label{sec:dbs_properties}

Here, we are concerned with the decision boundaries in the compressed feature space $\Psi$. For simplicity of notations, we rewrite the equation~\eqref{eq:psi_output} as $$z = W_\psi x_\psi + b_\phi,$$ where $W_\psi = U \Sigma$.

The decision boundary between any two classes $i$ and $j$ in the compressed feature space, denoted by $x^{f(i,j)}_\psi$ can be defined as the set of points satisfying the following two conditions
\begin{equation} \label{eq:db_con1}
    W_\psi [i] x^{cf(i,j)}_\psi +  b_\phi[i] = W_\psi [j] x^{f(i,j)}_\psi +  b_\phi[j]
\end{equation}
\begin{equation} \label{eq:db_con2}
    W_\psi [i] x^{cf(i,j)}_\psi +  b_\phi[i] \geq W_\psi [k] x^{cf(i,j)}_\psi +  b_\phi[k] \quad , \quad \forall k \notin \{ i,j \}
\end{equation}

These two equations define a subspace within the $\Psi$, but it is important to realize that this space may be empty. If the subspace is empty for the decision boundary between two specific classes, it means that those classes do not have an interface. 

Theoretically, it is possible for all class pairs to have an interface even when $n>>f$. 

The following is wrong. Talk about the dimensions of subspace. relate that to n and f.

Theoretically, as long as $n \leq f$, it is possible for each pair of classes to have an interface, i.e., existence of all decision boundaries is possible, but not necessarily guaranteed. However, when we have $f < n$, it may not be possible for some of the classes to have a direct interface with each other. Rather, the first class may be separated from the second classes via other intermediary classes. This is not necessarily undesirable as it would mean there is little chance of confusion between the two classes. For example, a model trained on Imagenet might have a decision boundary between two different types of fish, but perhaps, it may not have a decision boundary between the planes and alligators. Hence, it would be informative to understand and verify which classes do or do not have decision boundaries between them.

In an $f$-dimensional ambient space, the decision boundary between two classes may be a hyperplane of $f-1$ dimensions. These decision boundaries eventually intersect with each other as well as with the boundaries of the domain. The intersection of two or more hyperplanes may be a non-empty subspace of lower dimensions defining the decision boundary between multiple classes.

In our compressed feature space, dimensionality of feature space is always equal to the number of classes, i.e., $f=n$. We note that in practice almost all models have architectures with $f>n$ in their last hidden layer. There are some architectures that have $f<n$, e.g., certain variations of Swin Transformers, however, the accuracy of those architecture have proved have proved to be lower compared to accuracy of equivalent architectures with $f>n$.

A separate paper provides a theoretical analysis and discussion on the existence, intersections, and properties of decision boundaries in the feature space. Moreover, we provide a fast algorithm that can project the samples to the decision boundaries of trained models.

\subsubsection{Projecting a point to the decision boundaries}

Given a sample $x$ with classification $i = \mathcal{C}(x)$, we would be interested to project the $x$ to the closest point on the decision boundaries with other classes. This is interesting because it would tell us how far a given sample is located from the decision boundaries of other classes, and it would also give us the direction the sample has to take to reach those decision boundaries. Perhaps the direction that a sample has to take to reach a decision boundary is very different than the direction it has to take to reach the convex hull of training set. 

To project the point $x_\psi$ to the closest point on the decision boundary between classes $i$ and $j$, we minimize the following objective

\begin{equation} \label{eq:flip_objective}
    \min_{x^{cf(i,j)}_\psi} \| x^{cf(i,j)}_\psi - x_\psi \|_2
\end{equation}
subject to the constraints in equations~\eqref{eq:db_con1}-\eqref{eq:db_con2}. In our notation, $x^f$ denotes flip points while $x^{cf}$ denotes closest flip points (closest to some reference $x$). In practice, solving this optimization problem in a 1,000 dimensional feature space of a Swin Transformer trained on Imagenet takes about 50 milliseconds on a GPU.

\subsubsection{Special cases}

As we noted earlier, a direct interface might not exist between certain classes in which case the optimization problem above will become infeasible. One easy way to verify the infeasibility of our optimization problem is to solve its dual form. When the dual form is unbounded, the primal problem is guaranteed to be infeasible.

When the problem is infeasible, we can relax the formulation aiming for a proxy point that would tell us how far $x_\psi$ has to move in order to flip to the other class $j$. Since classes $i$ and $j$ do not have a direct interface, it means that $x_\psi$ has to flip from class $i$ to some other class(es) before eventually flipping to the class $j$. This implies a farther distance, hence, we seek to find a point on the interface with class $j$ and not necessarily with class $i$. For this, we rewrite the equations~\eqref{eq:db_con1}-\eqref{eq:db_con2} as

\begin{equation}
    W_\psi [j] x^{cf(i,j)}_\psi +  b_\phi[j]   \geq   W_\psi [k] x^{cf(i,j)}_\psi +  b_\phi[k] \quad , \quad \forall k \notin \{j\},
\end{equation}
and subject the objective function~\eqref{eq:flip_objective} to this relaxed constraint.

It is possible to generalize this to decision boundaries between the class $i$ and multiple classes specified in $\gamma$. In this case, we can minimize equation~\eqref{eq:flip_objective} subject to the following constraint
\begin{equation}
    W_\psi [i] x^{cf(i,\gamma)}_\psi +  b_\phi[i] = W_\psi [j] x^{f(i,\gamma)}_\psi +  b_\phi[j]  \quad , \quad \forall j \in \gamma
\end{equation}

\begin{equation}
    W_\psi [i] x^{cf(i,\gamma)}_\psi +  b_\phi[i] \geq W_\psi [k] x^{cf(i,\gamma)}_\psi +  b_\phi[k] \quad , \quad \forall k \notin \{ i,\gamma \}
\end{equation}

\subsubsection{Calculating the vector of distances}

To formalize the projection to the decision boundaries as a function, we use the following notation

\begin{equation} \label{eq:project_db}
    x^{cf(i,\gamma)}_\psi = \mathcal{P}^{f}(x_\psi,\gamma),
\end{equation}
where $\gamma$ denotes a set of classes that we want to find their joint decision boundary with class $i$.

Using the function in equation~\ref{eq:project_db}, we can find the distance between a point and the decision boundaries of all other $n-1$ classes. This ultimately gives us a vector of $n-1$ distances to each of the decision boundaries

\begin{equation}
    d^{f}_\psi(x_\psi): \forall k \in \{1:n \backslash i\} , \quad d^{f}_\psi(x_\psi)[k] =  \| x_\psi - \mathcal{P}^{f}(x_\psi,k) \|_F,
\end{equation}
while $d^{f}_\psi(x_\psi)[i] = 0$.


\section{Characterising the unknowns of the models}

Let us first characterize the notion of the unknowns of a trained model within its learned feature space.

\subsection{Boundaries of the domain in feature space}

In the pixel space, boundaries of the domain are well known. We can think of the pixel space as a hypercube of unit 1 dimensions as it is customary to normalize the pixel values to be between 0 and 1 or some other values. 

\begin{remark}
    While we consider the pixel domain as a continuous hypercube, we note that pixel space is usually a discrete space as pixel values have representations such as UINT8. In any case, the assumption of continuous space does not affect our conclusions as the discrete space will be a subset of continuous space.
\end{remark}




To realize the limits of the feature space $\Psi$, we can derive the following based on the Hölder's inequality
\begin{equation} \label{eq:bound_psi}
\begin{split}
    \| x_\psi \|_\infty &= \| V^T x_\phi \|_\infty \\ 
    & \leq \| V^T \|_{2,\infty}  \| x_\phi \|_2 \\
    & = \| x_\phi \|_2,
\end{split}
\end{equation}
since $V^T$ (from the SVD) is orthonormal, and therefore, $\| V^T \|_{2,\infty} = 1$.

Now, we have to identify the limits of the $\| x_\phi \|_2$ as the output of the $\mathcal{M}_\phi (.)$. For this, we can solve the following optimization problem
\begin{equation} \label{eq:bound_phi}
    \max_{x} \| \mathcal{M}_\phi(x) \|_2,
\end{equation}
subject to $x$ being within the boundaries of the pixel space. The maximum value of this objective function will give us the upper bound on $\| x_\phi \|_2$ which can then be plugged back into equation~\eqref{eq:bound_psi} to find the limit of the hypercube that defines our feature space $\Psi$. The hypercube is centered at the origin while each dimension of it is $2\| x_\psi \|_\infty$.

For solving the optimization problem in equation~\eqref{eq:bound_phi}, we note that $\mathcal{M}_\phi$ is a function of all the layers preceding the last layer. Depending on the architecture of the model, these layers could include ReLU layers, pooling layers, normalization layers, attention layers, etc. This composition of function as a whole is a non-concave function. Nevertheless, we can solve our maximization problem using standard optimization tools, perhaps the same algorithms that are used for training the models.

\begin{remark}
    To directly identify the limits of $\Phi$ (instead of $\Psi$), and avoid the compression performed via SVD, the norm of the objective function in equation~\eqref{eq:bound_phi} can be changed from the 2-norm to infinity norm.
\end{remark}

\begin{remark}
    For some specific models, the feature space $\Phi$ is strictly non-negative, e.g., when a model has a ReLU layer at the very end. Most models do not have that property and their last hidden layer may have positive and negative values.
\end{remark}


\subsection{Unknowns may constitute a large portion of the domain of a trained model} \label{sec: unknowns_large}

The arguments regarding the unknowns of the models are generally applicable to the pixel space as much as the feature space. In both spaces, the convex hull of training set usually  occupies a small fraction of the bounded domain. Testing samples are outside the \hull, yet their distance from the \hull usually does not exceed 30\% of the diameter of the \hull \cite{yousefzadeh2021hull}. For some models, this percentage is less than 5\% in the feature space.

At the same time, vast areas of the domain may be unrelated to the scope of a given model whether the domain is the pixel space or for other types of machine learning applications \cite{cao2023transparency}. For example, if we consider the pixel space for a model trained on the CIFAR-10 dataset, the same pixel space contains all sorts of images unrelated to the model's scope, e.g., radiology images, images of other object types, etc. All these irrelevant images exist in that same domain. What has become known as the adversarial inputs are also in that same domain.


To get a handle on how much of the domain is covered by the convex hull of training set, we estimate its volume using the approximation algorithm by \cite{simonovits2003compute}. This is one of the fastest approximation methods available \cite{lovasz2006simulated} and it makes our computation significantly faster since we are dealing with high dimensions and high number of samples. The running time of this algorithm is $\mathcal{O}(\bar{n}^4)$ where $\bar{n}$ is the number of dimensions.


\subsection{Existence of cases of improper confidence}

A simple example of improper over-confidence is the adversarial inputs which has been the subject of extensive, yet inconclusive, studies over the past few years. Another example is the case of out-of-distribution inputs where a model may confidently classify images that are unrelated to the scope of its training, e.g., an object detection model may classify a radiology image of liver as an aeroplane with 100\% confidence. First, let us provide a formulation which can quickly give us egregious cases of over confidence for a model:

\begin{equation}
    \max_{\hat{x}_\psi} \| x_\psi - \hat{x}_\psi \|_2
\end{equation}

subject to:
\begin{equation}
\begin{gathered}
    \hat{x}_\psi \in \Psi,\\ 
    \hat{x}_\psi \notin \mathcal{H}^{tr},\\ 
    softmax(U \Sigma \hat{x}_\psi + b_\phi)[i] > 0.9
\end{gathered}
\end{equation}

Solving this formulation is not hard, especially if we use a good initial solution that is inside the partition $i$ yet far from $x_\psi$. We leave the numerical examples for the future.










\subsection{Establishing the closeness to decision boundaries as the indicator of model's confidence}

In this section, we prove that only in the vicinity of decision boundaries a model has low confidence, and in all other regions of the domain, a model is guaranteed to have high confidence.

\begin{theorem} \label{theo:bound_delta_low}
    Given an input $x$ for a trained model $\mathcal{M}(.)$, if in the feature space of $\mathcal{M}$, the distance of $x$ from the closest decision boundary is at least $\delta$,
    $$\forall x \in \Phi: d^{f,min}_\phi (x) \geq \delta,$$
    then, the confidence (top softmax score) of $M(.)$ in classifying $x$ will be at least
    $$\max softmax(\mathcal{M}(x)) \geq \frac{ e^{\delta  \| W[j] - W[i] \|_2} }{ e^{\delta  \| W[j] - W[i] \|_2} + (n-1) },$$ where $W$ is the weight matrix used for the classification of the feature space, $n$ is the number of classes, $i$ is the classification of $x$, and $j$ is the class with closest decision boundary.
\end{theorem}

\begin{proof}
    We know $x$ is not on any decision boundaries, hence it will have a specific classification. Let us denote the classification of $x$ to be some class $i=\mathcal{C}(x)$. Without any loss of generality, we consider the feature space to be $\Phi$ while the same proof will apply to $\Psi$ as well.
    
    Let us denote the point on the closest decision boundary as $x_\phi^f$. Such decision boundary will be a flip between class $i$ and at least one other class. We denote the other class with $j$ while noting that $j$ may not be unique. From the theorem statement, we have $$\| v\|_2 = \delta \quad , \quad v = x_\phi^f - x_\phi.$$

    From equation~\eqref{eq:output_phi}, we can write $z = W x_\phi + b$ and $z^f = W x_\phi^f + b$. Plugging the definition of $v$ will give us $$z = W (x_\phi^f - v) + b = W x_\phi^f - W v +b = z^f - W v.$$    
    
    We have $z^f[i] = z^f[j]$ since $x^f$ is a flip point between classes $i$ and $j$. Based on this, we can write
\begin{equation} \label{eq:bound_delta}
\begin{split}
    z[i] - z[j] &= z^f[i] - z^f[j] - (W[i] v - W[j] v) \\
    &= 0 - (W[i] v - W[j] v) \\
    &= (W[j] - W[i]) v \\
    &= \| W[j] - W[i] \|_2 \| v \|_2 \cos{\theta} \\
    &= \delta  \| W[j] - W[i] \|_2,
\end{split}
\end{equation}
where $W[i]$ denotes the $i^{th}$ row of $W$, and $\theta$ is the angle between the vectors $W[j] - W[i]$ and $v$. It is easy to show that these two vectors are always parallel leading to $\cos{\theta} = 1$. To realize this, note that the decision boundary between classes $i$ and $j$ is a hyperplane with coefficients $W[j] - W[i]$. Since this hyperplane is the closest decision boundary to $x_\phi$, no other decision boundary can block the access of $x_\phi$ to this hyperplane, hence, $v$, the direction to the closest point on this hyperplane should be perpendicular to the hyperplane. Being perpendicular to a hyperplane means having the same direction as its vector of coefficients which makes $v$ parallel to $W[j] - W[i]$.

Now, via equation~\eqref{eq:bound_delta}, we have the difference between the elements $i$ and $j$ of the logit vector $z$. Hence, we can derive a bound on the softmax score (confidence) of $\mathcal{M}(x)$. The top softmax score for $x$, corresponding to class $i$, is
\begin{equation} \label{eq:bound_softmax}
\begin{split}
    \max softmax(x) &= \frac{e^{z[i]}}{\sum\limits_{\forall k \in \{1:n \} } e^{z[k]}} \\
    &\geq \frac{e^{z[i]}}{e^{z[i]} + (n-1) e^{z[j]}} \\
    &= \frac{e^{z[j]} e^{\delta  \| W[j] - W[i] \|_2} }{e^{z[j]} e^{\delta  \| W[j] - W[i] \|_2} + (n-1) e^{z[j]}} \\
    &= \frac{ e^{\delta  \| W[j] - W[i] \|_2} }{ e^{\delta  \| W[j] - W[i] \|_2} + (n-1) }
\end{split}
\end{equation}



\end{proof}

\begin{corollary}
    Given a specific model with known $W$, the bound in equation~\eqref{eq:bound_softmax} can be reduced to the following class agnostic bound
    \begin{equation}
        \max softmax(\mathcal{M}(x)) \geq \frac{ e^{\delta  \rho(W)} }{ e^{\delta  \rho(W) } + (n-1) },
    \end{equation}
    where 
    \begin{equation} \label{eq:rho_w}
        \rho(W) = \min_{i \neq j} \| W[j] - W[i] \|_2, \forall i,j \in \{1:n \}.
    \end{equation}
\end{corollary}

\begin{proof}
    Computing the $\rho(W)$ is straightforward via equation~\eqref{eq:rho_w}. This equation provides the smallest possible value for $\| W[j] - W[i] \|_2$ for any choice of $i \neq j$. We previously established that $i$ and $j$ cannot be identical. Furthermore, the function $\frac{e^x}{e^x+c}$ is monotonically increasing with respect to $x$ because its derivative w.r.t. $x$ is strictly positive. Therefore, the minimum value of $\| W[j] - W[i] \|_2$ will provide a lower bound for the right hand side of equation~\eqref{eq:bound_softmax}.
\end{proof}

We can extend the Theorem~\ref{theo:bound_delta_low} to derive an upper bound on the confidence of the model based on the distance of inputs from the decision boundaries.

\begin{theorem} \label{theo:bound_delta_up}
    Given an input $x$ for a trained model $\mathcal{M}(.)$, if in the feature space of $\mathcal{M}$, the distance of $x$ from the closest decision boundary is at most $\delta$,
    $$\forall x \in \Phi: d^{f,min}_\phi (x) \leq \delta,$$
    then, the confidence (top softmax score) of $M(.)$ in classifying $x$ will be at most
    \begin{equation}
        \max softmax(\mathcal{M}(x)) \leq \frac{ e^{\delta  \rho'(W)} }{ e^{\delta  \rho'(W) } + 1 },
    \end{equation}
    where 
    \begin{equation} \label{eq:rho_w_max}
        \rho'(W) = \max_{i \neq j} \| W[j] - W[i] \|_2, \forall i,j \in \{1:n \}.
    \end{equation}
\end{theorem}

\begin{proof}
    Proof is similar to the previous proofs and is left out for brevity.
\end{proof}

We can use the above theorems in two ways. First, we can consider a specific lower bound for the confidence and determine what distance should the samples maintain from the decision boundaries for a given confidence lower bound. This can then be used to estimate for what fraction of the domain a model has high confidence. Second, if we know how far the samples are from the decision boundaries, we can compute the minimum confidence of the model for those samples.

To make this relationship more tangible, let us consider two examples.

\begin{example}
    Consider a ViT model \cite{} trained on the CIFAR-10 dataset. We obtain the $\rho(W)$ for this model as $0.876$. Since the model has 10 classes, we have $n=10$. The relationship between the confidence of the model and the distance to the decision boundaries becomes $$ \max softmax(\mathcal{M}(x)) \geq \frac{ e^{0.876 \delta} }{ e^{0.876 \delta} + 9 }.$$ Feature space of this model is a 768-dimensional hypercube. To give a perspective about distances in the feature space of this model, we note that its diameter has size 110. The diameter of $\mathcal{H}^{tr}_{\phi}$ is 10.40. The average diatance to the decision boundaries is $5$. Given the bound above, choosing the distance threshold as $\delta = 3$ yields $$\max softmax(\mathcal{M}(x)) \geq 60\%.$$ This means that any sample that with distance to decision boundaries $\geq 3$ will be classified by the model with softmax score of at least 60\%. Alternatively, if we require the confidence to be at least 90\%, the distance from the decision boundaries should be at least $5.07$. Note that these are lower bounds, and there may exist points with high confidence that are closer to the decision boundaries. However, as long as the $5.07$ distance is maintained from the decision boundaries for any input, the confidence of the ViT model for the input is guaranteed to be at least 90\%.
    
    On the other hand, using Theorem~\ref{theo:bound_delta_up}, we obtain $\rho'(W) = 0.9692$ leading to $$ \max softmax(\mathcal{M}(x)) \leq \frac{ e^{0.9692 \delta} }{ e^{0.9692 \delta} + 1 },$$ which implies that for any sample with $d^{f,min}_\phi \leq 0.42$, the top softmax score of the model cannot be larger than 60\%.
\end{example}



\begin{remark}
    The theorem and the corollary proved above hold all over the domain inside and outside the convex hull of training set. Therefore, it is only the closeness to decision boundaries that can explain the low confidence of a model when classifying an input.
\end{remark}

\begin{remark}
    From equation~\eqref{eq:bound_delta}, we can conclude that the closest class to a given input is not necessarily the class with the second largest softmax score.
\end{remark}


\subsection{Characterizing the decision boundaries and the fraction of the domain occupied by them}

In previous section, we established that closeness to decision boundaries is the only source of low confidence for a model. We now aim to identify the decision boundaries and gauge the fraction of the domain occupied by them. A certain offset from the decision boundaries will indicate the fraction of the domain that a model would have low confidence.


\begin{theorem} \label{theo:volume_polytope}
    Given a hypercube $\Phi = [0,1]^{n}$ with $n \geq 2$ as the feature space of a classification model where $W$ and $b$ define the model's output via equation~\eqref{eq:output_phi}, the decision boundary between any two classes $i$ and $j$, denoted by $\Delta^{i,j}_\phi$, is a polytope of $\leq n-1$ dimensions. Volume of this polytope can be computed via Algorithm~\ref{alg:db_polytope} and it is upper bounded by 
    \begin{equation} \label{eq:volume_polytope}
    \begin{split}
        &vol(\Delta^{i,j}_\phi) \leq \\
        &\frac{\| W[i]-W[j] \|_2}{(n-1)! \prod\limits_{k=1}^n (W[i,k]-W[j,k])} \sum_{K \subseteq \{1,\dots, n \}} (-1)^{|K|} \big( \max ( 0 , -b[i] + b[j] - w . \mathbbm{1}_K ) \big)^{n-1}.
    \end{split}
    \end{equation}
\end{theorem}


\begin{proof}
    The decision boundary between any two arbitrary classes $i$ and $j$, is a subset of the domain defined by 
    \begin{equation} \label{eq:db_polytope}
    \begin{split}
        x &\in \Phi \\
         (W[i]-W[j])x &= -b[i] + b[j] \\
         W[i] x + b[i] &\geq W[k]x+b[k], \quad \forall k \notin \{i,j\}.
    \end{split}
    \end{equation}
    It is possible for this subset to be empty as we noted before in Section~\ref{sec:dbs_properties} in which case $\Delta^{i,j}_\phi$ would not occupy any portion of the $\Phi$. However, if the subset is non-empty, it will be a convex polytope of at most $n-1$ dimensions. Each vertex of this polytope will be an intersection of the $\Delta^{i,j}_\phi$ with either the boundaries of the $\Phi$ or with the decision boundaries of other classes. These intersections define the vertices of the polytope. To identify the vertices of the $\Delta^{i,j}_\phi$, we can use the simplex method in Algorithm~\ref{alg:db_polytope} inspired by the revised simplex algorithm in \cite{nocedal2006numerical}. This algorithm initially identifies a vertex of the feasible domain, then finds a direction where one of the constraints is binding, and moves in that direction until it finds the next vertex. This will continue until all the vertices are obtained as we know the number of vertices of this polytope is finite and upper bounded by $(n-[\frac{n}{2}])\binom{n}{\frac{n}{2}}$ \cite{o1971hyperplane}. Having all the vertices of $\Delta^{i,j}_\phi$, we can compute its volume analytically using formulas in \cite{lawrence1991polytope,bueler2000exact}.\footnote{One way to perform this computation is to break the polytope into a set of simplexes, and then compute the volume of each simplex using the Gram determinant.}\\
    If $\Delta^{i,j}_\phi$ only intersects with the boundaries of the $\Phi$ and does not have any intersections with other decision boundaries, $\Delta^{i,j}_\phi$ would be the intersection of a hyperplane and a hypercube, and its volume can be computed analytically via the right hand side of equation~\eqref{eq:volume_polytope} derived by \cite{ball1986cube,marichal2008slices}. However, if $\Delta^{i,j}_\phi$ intersects with any other decision boundaries, i.e., the last line of equation~\eqref{eq:db_polytope} becomes binding for one or more instances of $k$, then the volume of $\Delta^{i,j}_\phi$ will be strictly less than the right hand side of equation~\eqref{eq:volume_polytope}.
\end{proof}



\begin{algorithm}[H]
\caption{Identifying the decision boundaries in the feature space}\label{alg:db_polytope}
\begin{algorithmic}[1]
\Require model $\mathcal{M}(.)$ with $W$ and $b$ in its last hidden layer and $\Phi$ as its feature space, classes $i$ and $j$
\Ensure vertices of polytope $\Delta^{i,j}_\phi$ defining the decision boundary between classes $i$ and $j$ in the feature space of $\mathcal{M}(.)$
\State $\Phi = [-a , a]^n \leftarrow$ identify $a$ as the length of $\Phi$ by solving equation~\eqref{eq:bound_phi} with infinity norm
\State initialize $\mathcal{S} =\emptyset$
\State $\mathcal{S} \leftarrow$ identify $s$ as a basic feasible solution of equations~\eqref{eq:db_polytope}
\State $d \leftarrow$ identify a direction from $s$ along which one constraint in equation~\eqref{eq:db_polytope} is binding
\While{$d \neq \emptyset$}
    \State $\mathcal{S} \leftarrow$ move $s$ along $d$ as far as equations~\eqref{eq:db_polytope} are satisfied
    \State $d \leftarrow$ identify a direction from $s$ along which one new constraint in equation~\eqref{eq:db_polytope} is binding
\EndWhile
\end{algorithmic}
\end{algorithm}

\begin{remark}
    If $\Delta^{i,j}_\phi$ does not have any intersections with other decision boundaries, e.g., the case of binary classification, then the equation~\eqref{eq:volume_polytope} becomes an equality directly providing the volume of $\Delta^{i,j}_\phi$ with no need for Algorithm~\ref{alg:db_polytope}.
\end{remark}


Now that we can identify the exact location of decision boundaries and their volume, we can proceed to identify the slices of the domain surrounding the decision boundaries which are guaranteed to have high or low confidence.

\subsection{Regions of guaranteed high confidence in the feature space}

We are interested to identify the regions of the domain where a model is guaranteed to have high confidence, i.e., softmax score. Let us formally define such regions as $$\Phi(\geq\tau) \coloneqq \{x_\phi \: | \: x_\phi \in \Phi, \max softmax(\mathcal{M}(x_\phi)) \geq \tau \}.$$

In Theorems~\ref{theo:bound_delta_low}-\ref{theo:bound_delta_up}, we established that maintaining a certain distance from the decision boundaries will lead to lower bounds on the softmax score of any given model. Later, in Theorem~\ref{theo:volume_polytope}, we established the location of decision boundaries in the feature space and the corresponding volume that they occupy. It follows from these two sets of theorems that in the offsets of certain distance from the decision boundaries, a model is guaranteed to have low confidence, and in the regions away from those offsets, a model is guaranteed to have high confidence. We can formally identify the fraction of these regions using the following starting with the case of binary classification.

\begin{theorem} \label{theo:binary_confident}
    Given a model $\mathcal{M}(.)$ with feature space $\Phi = [-a,a]^n$ and two output classes, the fraction of feature space that is classified with softmax score of at least $\tau$ is lower bounded by
    \begin{equation} \label{eq:binary_volume}
        1 - \frac{\sqrt{2}}{a} \delta \leq 1 - \frac{2 \delta}{(2a)^n}  vol(\Delta_\phi^{i,j}) \leq vol(\Phi(\geq\tau) ),
    \end{equation}
    where the lower bound corresponds to the volume of all the points which maintain a minimum distance of
    \begin{equation} \label{eq:delta_confidence}
        \delta = \frac{1}{\rho(W)} \log\big(\frac{\tau(1-n)}{\tau-1}\big)
    \end{equation}
    from the decision boundaries in $\Phi$.
\end{theorem}

\begin{proof}
    The model has $2$ classes and one decision boundary, $\Delta_\phi^{i,j}$, dividing the $\Phi$ into two sections. The vertices of the polytope defining $\Delta_\phi^{i,j}$ can be identified via Algorithm~\ref{alg:db_polytope}, and its volume can be computed precisely via the right hand side of equation~\eqref{eq:volume_polytope}. 
    From Theorem~\ref{theo:bound_delta_up}, the offset of $\delta$ would guarantee classification with softwax score of at least $\tau$ for any point in the $\Phi$. The volume of the regions that have a minimum distance of $\leq \delta$ from the decision boundaries is less than or equal to the volume of decision boundaries times $2\delta$. Therefore, the remainder of the volume of the $\Phi$ is guaranteed to have softmax score $\geq \tau$.
    
    Moreover, $\Delta_\phi^{i,j}$ is the intersection of a hyperplane and a hypercube and its volume is upper bounded by $\sqrt{2} (2a)^{n-1}$ \cite{ball1986cube}. Replacing this for $vol(\Delta_\phi^{i,j})$ yields the leftmost side of the bound.
\end{proof}

\begin{remark}
    We note that the volume of the offset surrounding the $\Delta_\phi^{i,j}$ can be computed precisely by shifting the hyperplane between classes $i$ and $j$ by $\pm \delta$, and computing the corresponding vertices with the $\Phi$ resulting in a hypersimplex. The volume of the hypersimplex will be the exact value for $vol(\Delta_\phi^{i,j})$. Nevertheless, computing the volume of $\Delta_\phi^{i,j}$ and multiplying it with $2\delta$ may be a reasonably close approximation in the binary classification case.
\end{remark}



    

We now extend this to the multi-classification case. First, let us define a notation for the area surrounding any given decision boundary
\begin{equation}
    \Delta_\phi^{(i,j) \pm \delta} \coloneqq \{x_\phi \: | \: x_\phi \in \Phi, \frac{|(W[i] - W[j]) x_\phi + b[i] - b[j] |}{\| W[i] - W[j] \|_2}   \leq \delta    \},
\end{equation}

\begin{theorem} \label{theo:multi_confident}
    Given a trained model $\mathcal{M}(.)$ and its feature space $\Phi = [-a,a]^n$, the fraction of feature space that is classified with softmax score of at least $\tau$ is lower bounded by
    \begin{equation} \label{eq:total_highconf_ratio}
        \mathcal{R}(\Phi(\geq\tau) ) \geq 1 - \frac{1}{(2a)^n} vol (\Gamma \Delta) ,
    \end{equation}
    where $\Gamma \Delta$ denotes the union of all the decision boundaries and their offsets
    \begin{equation} \label{eq:total_volume_db}
        \Gamma \Delta =  \bigcup_{i=1}^{n-1}  \bigcup_{j=i+1}^n   \Delta_\phi^{(i,j) \pm \delta} ,
    \end{equation}    
    and $\delta$ is the offset from the decision boundaries defined by equation~\eqref{eq:delta_confidence}.
\end{theorem}

\begin{proof}
    From Theorem~\ref{theo:bound_delta_up}, the offset of $\delta$ from the decision boundaries guarantees classification with softwax score of at least $\tau$ for any point in the $\Phi$. A given model with $n$ output classes has $\binom{n}{2}$ distinct decision boundaries. For each decision boundary, the region that maintains a distance of $\leq \delta$ from it is defined by $\Delta_\phi^{(i,j) \pm \delta}$ and its computable via an extension of Algorithm~\ref{alg:db_polytope}. The union of these regions defines all the points in $\Phi$ that have a distance of $\leq \delta$ from at least one decision boundary. The complement of the volume of this union will be the lower bound on the desired fraction of feature space.
\end{proof}

\subsection{Regions of guaranteed overconfidence: High confidence for the unknowns}


In a separate study, we empirically establish that the scope of a model is tightly related to the convex hull of its training set in the feature space. 
For the purpose of the current study, we take it for granted that the scope of a trained model is tightly related to the convex hull of its training set
\begin{equation}
    x \in \Phi : \|x-\mathcal{H}^{tr} \|_2 \leq \delta,
\end{equation}
for some $\delta$ that is relatively small compared to the dimensions of $\mathcal{H}^{tr}$ and $\Phi$.

In section~\ref{sec: unknowns_large}, we discussed, based on empirical evidence, that the convex hull of training set may occupy a small fraction of the feature space and vast areas in the feature space may be unrelated to the scope of a given model. Now, we integrate these notions with the confidence of the models to identify the regions in the feature space that are unrelated to the scope, yet, the model is guaranteed to have high confidence.

To identify these regions, we incorporate the convex hull of training set into our previous theorems. In Theorems~\ref{theo:binary_confident}-\ref{theo:multi_confident}, we devised a method to identify the regions of the feature space that are guaranteed to be classified with high confidence. Now, we aim to identify the high confidence regions that maintain a certain distance from the convex hull of training set. Such regions would be the improper confidence of the trained model: overconfidence for the inputs irrelevant to the scope of the model.

We first approximate the convex hull of training set with the smallest hyperrectangle that contains it.

\begin{definition}
    Given the feature space $\Phi = [-a,a]^n$, and the convex hull $\mathcal{H}^{tr}_\phi \in \Phi$: we define  $\Phi^{\mathcal{H}}$ as the smallest hyperrectangle, with the same orientation as $\Phi$, that contains the $\mathcal{H}^{tr}$:
    $$\Phi^{\mathcal{H}} \coloneqq [l_i,u_i]_{i=1}^{f} \: | \: l_i = \min (x^j_\phi[i], \forall x_\phi^j \in \mathcal{D}^{tr}_\phi) , u_i = \max_j (x^j_\phi[i], \forall x_\phi^j \in \mathcal{D}^{tr}_\phi) $$
    Moreover, $\Phi^{\mathcal{H}+\delta^h}$ denotes the hyperrectangle with an outer offset of $\delta^h$ from $\Phi^{\mathcal{H}}$:
    $$\Phi^{\mathcal{H}+\delta^h} \coloneqq [l_i',u_i']_{i=1}^{f} \: | \: l_i' = \max (l_i - \delta^h, -a) , u_i = \min (u_i +\delta^h , a). $$
    It follows that any point outside the $\Phi^{\mathcal{H}+\delta^h}$ would have a distance of at least $\delta^h$ from the $\mathcal{H}^{tr}_\phi$. We consider the small offset of $\delta^h$ around the $\mathcal{H}^{tr}_\phi$ to still be relevant to the scope of the trained model. The value of $\delta^h$ shall be chosen empirically by verifying how far from the $\mathcal{H}^{tr}_\phi$ inputs can still be relevant to the scope of the model.
\end{definition}

We are now set to formalize the regions of the domain that maintain a distance of $\delta^h$ from $\mathcal{H}^{tr}_\phi$ (being unknown to the model), and yet, the model is guaranteed to have high confidence about any inputs in those regions. In the following, $\mathcal{R}^{cu}$ denotes the ratio of the volume of the domain that has such characteristics.

\begin{theorem} \label{theo:multi_confident_unknown}
    Given a model $\mathcal{M}(.)$ with feature space $\Phi = [-a,a]^n$ and two output classes, the fraction of feature space that maintains a minimum distance of $\delta^{h}$ from the $\mathcal{H}^{tr}_\phi$ and is classified with softmax score of at least $\tau$ is lower bounded by
    \begin{equation} \label{eq:multi_unknown_confident}
        \mathcal{R}^{cu} \ge 1 - \frac{1}{(2a)^n} \big( vol(\Phi^{\mathcal{H}} ) + vol(\Gamma \Delta ) - vol( \Phi^{\mathcal{H}} \cap \Gamma \Delta ) \big),
    \end{equation}
    where $\delta$ is defined in equation~\eqref{eq:delta_confidence} and $\delta^h$ is a small positive number defining the offset from the $\mathcal{H}^{tr}_\phi$.
\end{theorem}

\begin{proof}
    We know from Theorem~\ref{theo:bound_delta_low} that maintaining distance of $\delta$ from the decision boundaries guarantees classification with softmax score of at least $\tau$ anywhere in $\Phi$. $\Gamma \Delta$, calculated via equation~\eqref{eq:total_volume_db}, is the union of all regions that have a distance of $\le \delta$ from the decision boundaries in $\Phi$, i.e., regions that are classified with low confidence. All of these low confidence regions shall be excluded from the $\mathcal{R}^{cu}$ as they have softmax score $\le \tau$. The entire $\Phi^{\mathcal{H}}$ is also excluded from $\mathcal{R}^{cu}$ as they are close to the $\Phi^{\mathcal{H}}$ and considered known to the model. However, these two have an overlap which should not be deducted from the domain twice. $\Phi^{\mathcal{H}} \cap \Gamma \Delta$ is the portion of $\Gamma \Delta$ that intersects with $\Phi^{\mathcal{H}}$ which we add back to $\mathcal{R}^{cu}$. The remainder of the $\Phi$ is unknown to the model, yet, the model is guaranteed to have high confidence.
\end{proof}

Previously, we showed how each of these components can be computed. Later, we will use these formulations to evaluate the percentage of domains that are unknown to the models yet the models, unjustifiably have high confidence about them. We see that these regions are considerably large.











\section{Unknowns of the trained models in relation to common failure modes}

Here, we consider four specific modes of failure, two of which fall under the umbrella of confusion.

\subsection{Confusion between classes}

By confusion, we refer to cases where a model may be undecided about which class applies to a given input. One such confusion would occur when a point is on or close to a decision boundary. Another, more complicated case will arise when we encounter overlapping classes. That is, when we have two or more classes where their training set overlap each other. The overlap region in particular would be a source of confusion. Such overlaps appear less often in computer vision tasks, but they are quite common in social applications of AI, e.g., the Adult Income dataset \cite{Dua2017uci}. To quantify the confusion for overlapping classes, we consider the distance to the convex hull of classes.

\subsubsection{Separable classes}

The first type of confusion is when training samples of each class are separable from other classes in the feature space. In such cases, closeness to the decision boundary that has separated the classes would be an indication for a model's confusion. For this, we measure the distance to the closest decision boundary
\begin{equation} \label{eq:amb_min_db}
    d^{f,min}_\psi (x_\psi) = \min_{j \in \{1:n \backslash i \}} d^{f(i,j)}_\psi (x_\psi)
\end{equation}
This measure is inversely related to ambiguity of an input.

\subsubsection{Overlapping classes}

In some cases, we encounter datasets where classes have some overlap. Such overlap could be in the original domain and/or in the feature space. To reflect this in our ambiguity measure, we consider the difference between the distance to the two closest convex hulls. In other words, we first sort the vector $d^{\mathcal{H}}_\psi (x)$, then subtract the first element of the sorted vector from its second element
\begin{equation} \label{eq:amb_dif_hull}
    d^{c}_\psi (x_\psi) = d^{\mathcal{H},sorted}_\psi (x) [1] - d^{\mathcal{H},sorted}_\psi (x)[0]
\end{equation}

This measure is inversely related to the ambiguity. Because the vector, $d^{\mathcal{H},sorted}_\psi (x)$, is sorted, its first element is not larger than its second element, and hence, we have $d^{c}_\psi (x) \geq 0$ for any input.

When there is an overlap between two classes of training set, and a sample falls into that overlap, its distance from the convex hull of each class will be zero, and the resulting $d^{c}_\psi (x)$ will be zero as well, leading to $\zeta(x) \rightarrow \infty$.

\begin{figure}[h]
     \centering
     \begin{subfigure}[b]{0.48\textwidth}
         \centering
         \includegraphics[width=1\linewidth]{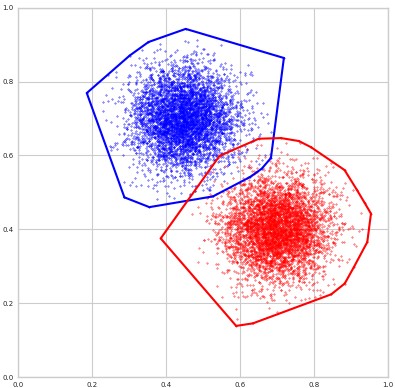}
         \caption{}
         \label{fig:y equals x}
     \end{subfigure}
     \hfill
     \begin{subfigure}[b]{0.48\textwidth}
         \centering
         \includegraphics[width=1\linewidth]{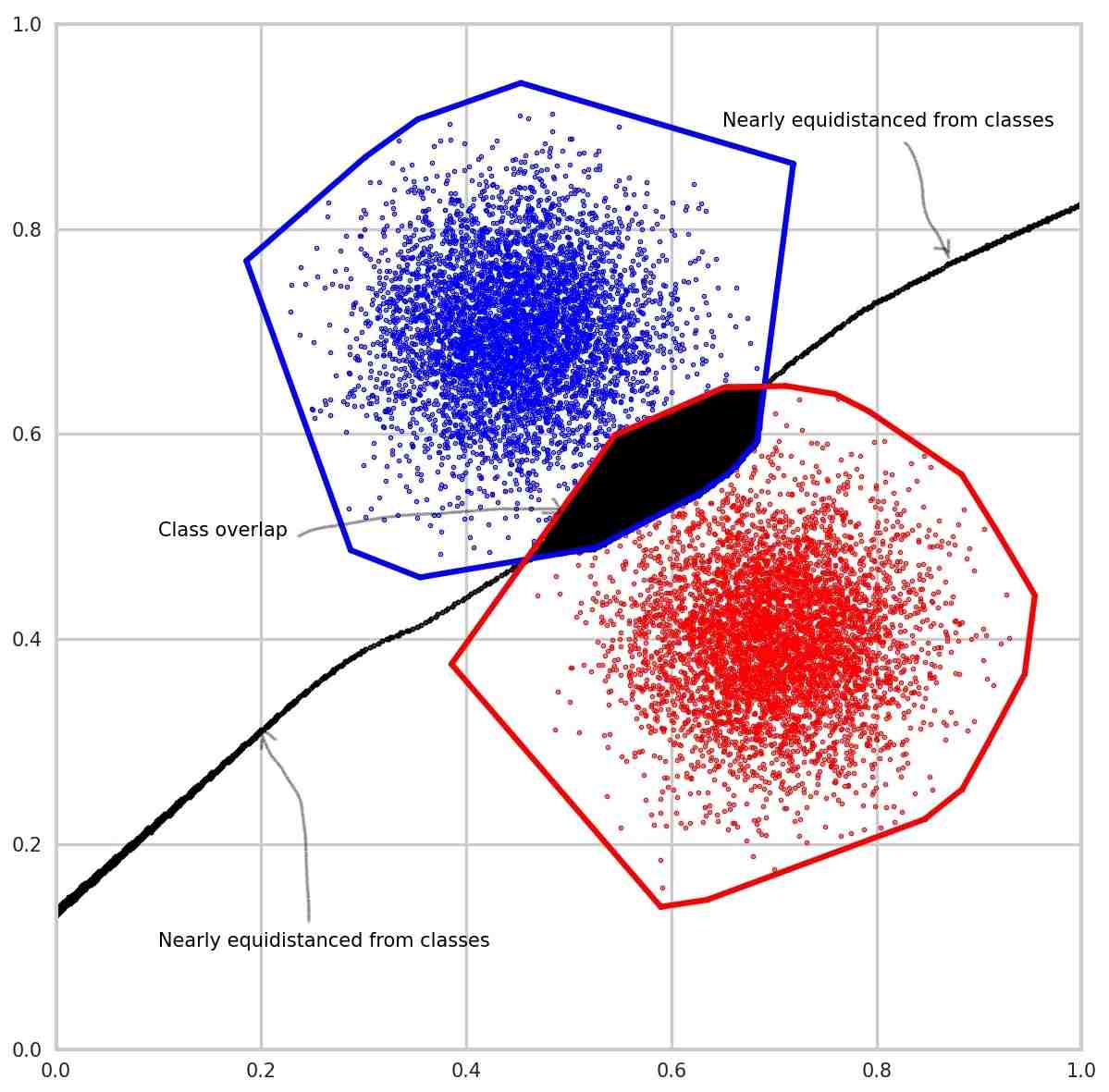}
         \caption{}
         \label{fig:three sin x}
     \end{subfigure}
        \caption{\textbf{(a)} Example of a 2D domain with overlapping classes and \textbf{(b)} the regions where $d^{c}_\psi (x_\psi)$ is equal or close to zero leading to high ambiguity. In the overlap region, we have $d^{c}_\psi (x_\psi)=0$ causing the ambiguity measure to go to infinity. For points in the domain that are equidistanced from the convex hull of these two classes, we again have $d^{c}_\psi (x_\psi) = 0$. For all other points in the domain, $d^{c}_\psi (x_\psi)$ is strictly positive. $d^{c}_\psi (x_\psi)=0$ would be larger for points that are close to one class and not as close to the other class, implying little or no confusion.}
        \label{fig:2d_overlap}
\end{figure}

One could consider limiting this measure to the samples that are inside the overlapping regions, i.e., disregard it whenever a sample does not coincide with an overlapping region. However, we find this measure to be insightful even when a sample is outside the convex hull of classes. Outside the convex hull of classes, this measure will go to zero when a sample is equidistant from the two closest classes as shown in Figure~\ref{fig:2d_overlap}. This, again, may imply confusion between the two classes. In our ablation studies, we will demonstrate how insightful this measure can be. Moreover, we show that instead of computing the difference between the two distances in equation~\eqref{eq:amb_dif_hull}, one can use a ratio of distances.

\subsection{Excessive extrapolation}

This case occurs when a model processes an input that excessively outside the convex hull of training set. We measure this via the minimum distance to the convex hull of classes
\begin{equation} \label{eq:amb_min_hull}
    d^{\mathcal{H},min}_\psi (x_\psi) = \min_{j \in \{1:n \}} d^{\mathcal{H},j}_\psi (x_\psi).
\end{equation}
This measure is directly related to ambiguity.

\subsection{Gaps in the model's knowledge}

Geometrically, this type of failure refers to the possible holes that may exist within the training samples. For example, consider a model that makes clinical decisions for patients. The training set of such model may only consist of patients ages 5-10 and patients aged 80-90 years old. When this model is deployed to make a decision for a patient that is 40 years old, it might be interpolating between its training samples, yet, in the dimension of age, the 40 years old testing sample may be too far from any of the training samples. We characterize this notion of knowledge gap by identifying the ball with the largest radius centered at the sample $x_\psi$ not containing any of training samples. The radius of such ball is our proxy for the knowledge gap:
\begin{equation} \label{eq:amb_gap_hull}
    d_\psi^g (x_\psi) = \max_\rho \rho: \forall y_\psi \in \mathcal{D}^{tr}_\psi: \| x_\psi - y_\psi \|_2 \geq \rho  
\end{equation}
This measure is directly related to ambiguity.






\subsection{Our ambiguity measure}

We are now ready to incorporate all of the above geometric information into a unified ambiguity measure.

\begin{equation} \label{eq:amb_comp}
    \zeta (x) =  \Big(\frac{(d^{\mathcal{H},min}_\psi(x) + \epsilon) (d_\psi^g (x) + \epsilon)}{d^{f,min}_\psi(x) d^{c}_\psi (x)}\Big)^\alpha,
\end{equation}
where $\epsilon$ and $\alpha$ are positive scalars to normalize the ambiguity measure. The purpose of $\epsilon$ is to extend the ambiguity measure inside the convex hull of training set. When we encounter a sample that falls inside the \hull (while not coinciding with any of the training samples), its corresponding $d^{\mathcal{H},min}_\psi$ will be zero, yet the other three measures can reflect useful information in $\zeta$. Our recommendation is to choose $\epsilon \ll 1$. For the $\alpha$, it can be chosen such that ambiguity is bounded between 0 and 1 for a set of samples.

\begin{remark}
The order of $\zeta(.)$ for a set of distinct inputs is preserved for any choice of $\epsilon > 0$ and $\alpha > 0$.
\end{remark}

The measure provided by equation~\eqref{eq:amb_comp} is well defined even in its extremes. We know that all distances are lower bounded by zero, therefore, the value of $\zeta(.)$ for any input in the domain is non-negative. When $d^{g}_\psi \rightarrow 0$, it follows that $d^{\mathcal{H},min}_\psi \rightarrow 0$, and as a result, the ambiguity measure goes to $\epsilon^2$ indicating that the input is not ambiguous. As the distance to the convex hull increases, the ambiguity increases linearly. By contrast, when $d^{f,min}_\psi \rightarrow 0$ and/or $d^{c}_\psi \rightarrow 0$, the ambiguity goes to infinity, i.e., the input is highly ambiguous. We almost never deal with these extreme cases, and almost always, these distances are positive values that are not too small or too large.

{\bf Models trained with special training methods:}
An underlying assumption here is that training set contains samples with hard labels. There are training methods, such as Mixup \cite{zhang2018mixup} that trains using soft labels. Such methods define the location of decision boundaries in between the samples. This does not contradict our approach because in such instances, we still have the original training samples with their hard labels, and moreover, the interpolated samples with soft labels are not visually meaningful, i.e., they only define the location of decision boundaries.

{\bf Models trained with unlabeled data:}
In self-supervised learning and in contrastive learning, it is possible to train a model on training samples without labels. Such methods are successful because the number of classes are large and the probability of encountering false positive samples is very small, i.e., any two random pair of samples that one picks from an unlabeled training set are unlikely to be from a same class. Such models, trained with contrastive learning, need to be fine-tuned eventually with labeled training samples in order to define the location of decision boundaries. This last stage will provide the frame of reference we need to compute our ambiguity measure. In other words, when a model is fine-tuned and the location of its decision boundaries are defined, we can go back to the unlabeled training set and label the samples using the fine-tuned model. Once we obtain the labeled training set, we can proceed with computing our ambiguity measure for any testing sample.

\subsection{A simplified version of ambiguity measure: rule of thumb}

As a rule of thumb alternative to the measure provided in equation~\eqref{eq:amb_comp}, one can consider using the following simplified measure
\begin{equation} \label{eq:thumb_amb}
    \bar{\zeta} (x) = \Big(\frac{d^{\mathcal{H},min}_\psi(x) + \epsilon}{d^{f,min}_\psi(x)}\Big)^\alpha.
\end{equation}
We show in our experiments that this rule of thumb works quite effectively in performing a wide range of tasks regarding the model failures albeit not as good as the equation~\ref{eq:amb_comp}.

We have considered other rules of thumb as well, e.g., subtracting the two distances of $d^{\mathcal{H},min}_\psi(x)$ and $d^{f,min}_\psi(x)$ instead of dividing them. Overall, we did not see a significant difference in predictive power of the measure in our empirical experiments. We chose the ratio instead of the difference, because it appeared to be more intuitive and slightly better in its predictive power.



\subsection{Ambiguity explanation} \label{sec:explanation}

We can generate an automatic text explaining why an input is ambiguous or not. Generally, this text will describe each component of the ambiguity measure in a percentile form in comparison to other samples.

The structure of the text is as following:

    \textit{"Classification of this input is \rule{1cm}{0.15mm}. This image is \underline{ambiguous/unambiguous} because its ambiguity measure is \underline{higher/lower} than \rule{.5cm}{0.15mm} \% of all samples. The ambiguity of this input has four components:}
    \begin{itemize}
    \itemsep0em
        \item Its distance to the convex hull of training set is > \rule{.5cm}{0.15mm} percentile
        \item Its distance to closest decision boundary > \rule{.5cm}{0.15mm} percentile
        \item It is close to both classes \rule{.5cm}{0.15mm} and \rule{.5cm}{0.15mm}
        \item The radius of gap around this sample > \rule{.5cm}{0.15mm} percentile
    \end{itemize}









\subsection{Abstention model}

At the deployment time, for any given input, the model first computes the ambiguity measure. If the ambiguity measure is larger than the chosen threshold, the model abstains from classifying. Otherwise, it will return the classification of input along with the ambiguity measure.

\begin{algorithm}[H]
\caption{Unknown aware inference}\label{alg:amb_inference}
\begin{algorithmic}[1]
\Require model $\mathcal{M}(.)$, input $x$, ambiguity function $\zeta(.)$, ambiguity threshold $\tau$
\Ensure Classification of $x$
\State Calculate $x_\psi$ using equations~\eqref{eq:map_phi} and~\eqref{eq:map_psi}
\State Compute $\zeta(x)$ using equation~\eqref{eq:amb_comp}
\If{$\zeta(x) > \tau$}
    \State \textit{return} unsure: abstaining from classification
    \State \textit{return} ambiguity explanation (Section~\ref{sec:explanation})
\Else
    \State \textit{return} classification: $\mathcal{M}(x)$, ambiguity measure: $\zeta(x)$
\EndIf
\end{algorithmic}
\end{algorithm}

The threshold $\tau$ can be chosen based on the statistics of ambiguity measure computed for a set of samples perhaps from a validation set. For example, one can compute the ambiguity of inputs on a validation set, and then set the $\tau$ that best separates the misclassifications of that validation set from its correct classifications. Alternatively, for any model that is already used in practice for a while, there will be some samples of its failure modes available at hand. One can compute the ambiguity of those failures and set the $\tau$ that best identifies those failures. In choosing the $\tau$, limiting the rate of false positives may also be a consideration.

If we do not have access to a validation set or any samples of previous failures, we can compute the ambiguity measure for the training samples, and then choose some threshold based on the ambiguity distribution. For example, we may set the threshold as the 99th percentile of the ambiguity of training samples.

\subsection{Ambiguity model}

To build an ambiguity model, we may train a classifier on the geometric information of each point in the feature space. The training objective could be to detect undesirable inputs that the model may be exposed to, or the cases where the model is likely to make a mistake. In other words, such ambiguity model will aim to detect and weed out the inputs that the model may not be fit to process them.

Algorithm~\ref{alg:amb_model} presents a formal procedure to train such a model. This procedure relies on having access to a set of normal in-distribution samples $\mathcal{D}^{p}$, and a set of undesirable samples that the model is likely to fail on them $\mathcal{D}^{n}$. The set of undesirable samples in $\mathcal{D}^{n}$ could consist of any undesirable samples that we might have available. For example, if $\mathcal{M}$ has been deployed in practice, it might have made some mistakes already which we can include in $\mathcal{D}^{n}$. Or if we have encountered some OOD or adversarial inputs, we can include them. If we are concerned about specific types of OOD inputs or certain types of adversarial attacks, such samples could be gathered/generated and then included in $\mathcal{D}^{n}$ as well.

As in the case of abstention models, we rely on the geometric information available in the feature space. For any given input, we can compute the geometric information using equations~\eqref{eq:amb_min_db}-\eqref{eq:amb_gap_hull}. Assembling these measurements together as a unified vector for each input in $\mathcal{D}^{p}$ and $\mathcal{D}^{n}$ leads to two new datasets: $\mathcal{G}^{p}$ and $\mathcal{G}^{n}$ which only reflect the geometric information for each of the samples.

About the $\mathcal{A}(.)$, one can use any appropriate classification model. Our task, here, is binary classification and we do not expect $\mathcal{G}_\psi^{n}$ to have high dimensions. Clearly, $\mathcal{G}_\psi^{p}$ and $\mathcal{G}_\psi^{n}$ can be split into a training/validation/testing portions when we train the $\mathcal{A}(.)$. We have experimented with models such as XGBoost, Random Forests, and SVMs, most of which perform reasonably well. 
We train ambiguity models to detect a variety of failure modes such as detecting corner cases, misclassifications, adversarial inputs, and OOD.

\begin{algorithm}[H]
\caption{Creating an ambiguity model to detect undesirable (ambiguous/adversarial/OOD) inputs}\label{alg:amb_model}
\begin{algorithmic}[1]
\Require model $\mathcal{M}$, set of normal samples $\mathcal{D}^{p}$, set of undesirable samples $\mathcal{D}^{n}$
\Ensure ambiguity detection model $\mathcal{A}$
\State $\mathcal{D}^{p}_\psi, \mathcal{D}^{n}_\psi \leftarrow$ map all samples of $\mathcal{D}^{p}$ and $\mathcal{D}^{n}$ to feature space $\Psi$
\State $\mathcal{G}_\psi^{p} \leftarrow $ compute $d^{\mathcal{H},min}_\psi(.)$, $d_\psi^g (.)$, $d^{f,min}_\psi(.)$, and $d^{c}_\psi (.)$ for $\mathcal{D}^{p}_\psi$
\State $\mathcal{G}_\psi^{n} \leftarrow $ compute $d^{\mathcal{H},min}_\psi(.)$, $d_\psi^g (.)$, $d^{f,min}_\psi(.)$, and $d^{c}_\psi (.)$ for $\mathcal{D}^{n}_\psi$
\State train a model $\mathcal{A}(.)$ to classify $\mathcal{G}_\psi^{p}$ from $\mathcal{G}_\psi^{n}$
\State \textit{return} $\mathcal{A}(.)$
\end{algorithmic}
\end{algorithm}

We note that in the literature, there are models specifically designed for adversarial detection that are trained merely on the inputs in the feature space \cite{moayeri2021sample}. That would be equivalent to omitting lines 2-3 of the algorithm above, training the detection model not on $\mathcal{G}_\psi^{p}$ and $\mathcal{G}_\psi^{n}$, rather training the model on $\mathcal{D}_\psi^{p}$ and $\mathcal{D}_\psi^{n}$. We will see in our experimental results that computing the geometric arrangements enables the model to identify a variety of undesirable inputs (not just adversarial) better than merely from the projections to the feature space.

\section{Results}

In this section, we present our results on a ResNet model pre-trained on the CIFAR-10 dataset and a Swin Transformer pretrained on Imagenet. In the appendix, we provide results obtained from other pretrained models as well.

To be added. Stay tuned.

\clearpage

\small
\bibliographystyle{plain}
\bibliography{main}

\end{document}